\documentclass{article}

\usepackage{times}
\usepackage{graphicx}
\usepackage{subfigure} 
\usepackage{natbib}
\usepackage{algorithm}
\usepackage{algorithmic}
\usepackage{hyperref}
\usepackage{amsmath}
\usepackage{amssymb}

\newcommand{\diag}{\mathop{\mathrm{diag}}}
\newcommand{\fft}{\mathfrak{F}}
\usepackage{booktabs}
\usepackage{pbox}
\usepackage[backgroundcolor=green!20]{todonotes}
\usepackage{tabularx}
\usepackage{tikz}

\usepackage{amsthm}
\newtheorem{theorem}{Theorem}[section]
\newtheorem{corollary}{Corollary}[theorem]
\newtheorem{lemma}[theorem]{Lemma}
\newtheorem{definition}{Definition}[section]
\newtheorem{proposition}{Proposition}[section]

\usepackage[accepted]{icml2017}

\icmltitlerunning{Analogical Inference for Multi-relational Embeddings}

\begin{document} 

\twocolumn[
\icmltitle{Analogical Inference for Multi-relational Embeddings}

\begin{icmlauthorlist}
\icmlauthor{Hanxiao Liu}{cmu}
\icmlauthor{Yuexin Wu}{cmu}
\icmlauthor{Yiming Yang}{cmu}
\end{icmlauthorlist}

\icmlaffiliation{cmu}{Carnegie Mellon University, Pittsburgh, PA 15213, USA}

\icmlcorrespondingauthor{Hanxiao Liu}{hanxiaol@cs.cmu.edu}

\vskip 0.3in
]

\printAffiliationsAndNotice{}

\begin{abstract} 
Large-scale multi-relational embedding refers to the task of learning the latent representations for entities and relations in large knowledge graphs. An effective and scalable solution for this problem is crucial for the true success of knowledge-based inference in a broad range of applications. This paper proposes a novel framework for optimizing the latent representations with respect to the \textit{analogical} properties of the embedded entities and relations. By formulating the learning objective in a differentiable fashion, our model enjoys both theoretical power and computational scalability, and significantly outperformed a large number of representative baseline methods on benchmark datasets. Furthermore, the model offers an elegant unification of several well-known methods in multi-relational embedding, which can be proven to be special instantiations of our framework.
\end{abstract} 

\section{Introduction}
\label{sec:introduction}

Multi-relational embedding, or knowledge graph embedding, is the task of finding the latent representations of entities and relations for better inference over knowledge graphs.  This problem has become increasingly important in recent machine learning due to the broad range of important applications of large-scale knowledge bases,
such as Freebase \cite{bollacker2008freebase}, DBpedia \cite{auer2007dbpedia} and Google's Knowledge Graph \cite{singhal2012introducing},
including 
question-answering \cite{ferrucci2010building}, information retrieval \cite{dalton2014entity} and
natural language processing \cite{gabrilovich2009wikipedia}.

A knowledge base (KB) typically stores factual information as subject-relation-object triplets. 
The collection of such triplets forms a directed graph whose nodes are entities and whose edges are the relations among entities.
Real-world knowledge graph is both extremely large and highly incomplete by nature \cite{min2013distant}.  How can we use the observed triplets in an incomplete graph to induce the unobserved triples in the graph presents a tough challenge for machine learning research.  
\begin{figure}[t]
    \centering
    \begin{tikzpicture}[
    scale=0.8, every node/.style={scale=0.8},
      back line/.style={dashed},
      cross line/.style={preaction={draw=white, -,line width=6pt}}]
      \node (A) {\textcolor{blue}{$sun$}};
      \node [right of=A, node distance=5cm] (B) {\textcolor{blue}{$planets$}};
      \node [below of=B, node distance=2.5cm] (D) {\textcolor{blue}{$mass$}};

      \node (A1) [right of=A, above of=A, node distance=2.5cm] {\textcolor{red}{$nucleus$}};
      \node [right of=A1, node distance=5cm] (B1) {\textcolor{red}{$electrons$}};
      \node [below of=B1, node distance=2.5cm] (D1) {\textcolor{red}{$charge$}};

      \draw[->, back line] (A1) to node [below, sloped] {$attract$} (D1);
      \draw[->, cross line] (A) to node [below, sloped] {$attract$} (D);
      \draw[->, cross line] (A) to node [above] {$surrounded\_by$} (B);
      \draw[->, cross line] (A1) to node [above] {$surrounded\_by$} (B1);
      \draw[->, cross line] (B) to node [below, sloped] {$made\_of$} (D);
      \draw[->, cross line] (B1) to node [below, sloped] {$made\_of$} (D1);
      \draw[->, cross line] (A) to node [above, sloped] {$scale\_down$} (A1);
      \draw[->, cross line] (B) to node [above, sloped] {$scale\_down$} (B1);
      \draw[->, cross line] (D) to node [above, sloped] {$scale\_down$} (D1);
    \end{tikzpicture}
    \caption{Commutative diagram for the analogy between the Solar System (\textcolor{red}{red}) and the Rutherford-Bohr Model (\textcolor{blue}{blue}) (atom system).
    By viewing the atom system as a ``miniature'' of the solar system (via the $scale\_down$ relation),
    one is able to complete missing facts (triplets) about the latter by mirroring the facts about the former.
        The analogy is built upon three basic analogical structures (parallelograms):
        ``$sun$ is to $planets$ as $nucleus$ is to $electrons$'', ``$sun$ is to $mass$ as $nucleus$ is to $charge$'' and ``$planets$ are to $mass$ as $eletrons$ are to $charge$''.
    }
    \label{fig:solar}
\end{figure}
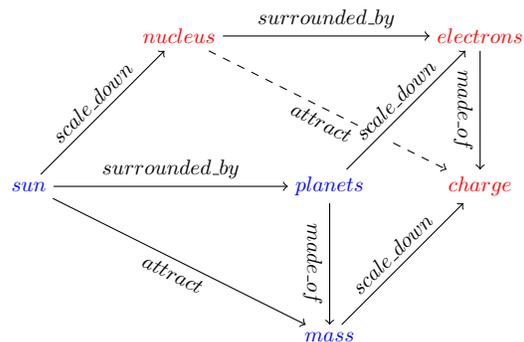

Various statistical relational learning methods \cite{getoor2007introduction, nickel2015review}
have been proposed for this task,
among which vector-space embedding models are most particular due to their advantageous performance and scalability \cite{bordes2013translating}. The key idea in those approaches
is to find dimensionality reduced representations for both the entities and the relations, and hence 
force the models to generalize during the course of compression.
Representative models of this kind include tensor factorization
\cite{singhal2012introducing, nickel2011three},
neural tensor networks \cite{socher2013reasoning, chen2013learning},
translation-based models \cite{bordes2013translating, wang2014knowledge, lin2015learning},
bilinear models and its variants \cite{DBLP:journals/corr/YangYHGD14a, DBLP:conf/icml/TrouillonWRGB16},
pathwise methods \cite{guu2015traversing}, embeddings based on holographic representations \cite{DBLP:conf/aaai/NickelRP16}, and product graphs that utilizes additional site information for the predictions of unseen edges in a semi-supervised manner \cite{liu2015bipartite,liu2016cross}.
Learning the embeddings of entities and relations can be viewed as a knowledge induction process, as those induced latent representations can be used to make inference about new triplets that have not been seen before.

Despite the substantial efforts and great successes so far in the research on multi-relational embedding, one important aspect is missing, i.e., to study the solutions of the problem from the \textit{analogical inference} point of view, by which we mean to rigorously define the desirable analogical properties for multi-relational embedding of entities and relations, and to provide algorithmic solution for optimizing the embeddings w.r.t.\ the analogical properties.
We argue that analogical inference is particularly 
desirable for knowledge base completion,
since for instance if system $A$ (a subset of entities and relations) is analogous to system $B$ (another subset of entities and relations),
then the unobserved triples in $B$ could be inferred by mirroring their counterparts in $A$. Figure \ref{fig:solar} uses a toy example to illustrate the intuition, where system $A$ corresponds to the solar system with three concepts (entities), and system $B$ corresponds the atom system with another three concepts. An analogy exists between the two systems because $B$ is a ``miniature'' of $A$. As a result, knowing how the entities are related to each other in system $A$ allows us to make inference about how the entities are related to each other in system $B$ by analogy.

Although \textit{analogical reasoning} was an active research topic in classic AI (artificial intelligence), early computational models mainly focused on non-differentiable rule-based reasoning \cite{gentner1983structure, falkenhainer1989structure, turney2008latent}, which can hardly scale to very large KBs such as Freebase or Google's Knowledge Graph.  How to leverage the intuition of analogical reasoning via statistical inference for automated embedding of very large knowledge graphs has not been studied so far, to our knowledge. 

It is worth mentioning that analogical structures have been observed in the output of several word/entity embedding models \cite{mikolov2013distributed, pennington2014glove}. However, those observations stopped there as merely empirical observations. Can we mathematically formulate the desirable analogical structures and leverage them in our objective functions to improve multi-relational embedding? In this case, can we develop new algorithms for tractable inference for the embedding of very large knowledge graphs?
These questions present a fundamental challenge which has not been addressed by existing work, and answering these questions are the main contributions we aim in this paper. We name this open challenge as the \textit{analogical inference} problem, for the distinction from rule-based \textit{analogical reasoning} in classic AI.

Our specific novel contributions are the following:
\begin{enumerate}
    \item A new framework that, for the first time,
        explicitly models analogical structures in multi-relational embedding,
        and that improves the state-of-the-art performance on benchmark datasets;
    \item The algorithmic solution for conducting analogical inference in a differentiable manner, whose implementation is as scalable as the fastest known relational embedding algorithms;
    \item The theoretical insights on how our framework provides a unified view of several representative methods as its special (and restricted) cases, and why the generalization of such cases lead to the advantageous performance of our method as empirically observed.
\end{enumerate}

The rest of this paper is organized as follows:
\S \ref{sec:linearmaps} introduces related background where multi-relational embedding is formulated as linear maps.
\S \ref{sec:analogies} describes our new optimization framework where the desirable analogical structures are rigorously defined by the the commutative property of linear maps.
\S \ref{sec:algorithm} offers an efficient algorithm for scalable inference by exploiting the special structures of commutative linear maps,
\S \ref{sec:unifiedview} shows how our framework subsumes several representative approaches in a principled way,
and \S \ref{sec:experiments} reports our experimental results,
followed by the concluding remarks in \S \ref{sec:conclusion}.

\section{Related Background}
\label{sec:linearmaps}

\subsection{Notations}
Let $\mathcal{E}$ and $\mathcal{R}$ be the space of all entities and their relations.
A knowledge base $\mathcal{K}$ is a collection of triplets $(s, r, o) \in \mathcal{K}$
where $s \in \mathcal{E}, o \in \mathcal{E}, r \in \mathcal{R}$ stand for the subject, object and their relation, respectively.
Denote by $v \in \mathbb{R}^{|\mathcal{E}| \times m}$ a look-up table where $v_e \in \mathbb{R}^m$ is the vector embedding for entity $e$,
and denote by tensor $W \in \mathbb{R}^{|\mathcal{R}| \times m \times m}$ another look-up table where $W_r \in \mathbb{R}^{m \times m}$ is the matrix embedding for relation $r$.
Both $v$ and $W$ are to be learned from $\mathcal{K}$.

\subsection{Relations as Linear Maps}
We formulate each relation $r$ as a linear map that,
for any given $(s, r, o) \in \mathcal{K}$, transforms the subject $s$ from its original position in the vector space to somewhere near the object $o$.
In other words,
we expect the latent representations for any valid $(s, r, o)$ to satisfy
\begin{equation}
    v_s^\top W_r \approx v_o^\top
    \label{eq:multi}
\end{equation}
The degree of satisfaction in the approximated form of  \eqref{eq:multi} can be quantified using
the inner product of $v_s^\top W_r$ and $v_o$.
That is, we define a bilinear score function as:
\begin{equation}
    \phi(s,r,o) = \langle v_s^\top W_r, v_o \rangle = v_s^\top W_r v_o
\end{equation}
Our goal is to learn $v$ and $W$ such that
$\phi(s, r, o)$ gives high scores to valid triples, and low scores to the invalid ones. 
In contrast to some previous models \cite{bordes2013translating}
where relations are modeled as additive translating operators,
namely $v_s + w_r \approx v_o$,
the multiplicative formulation in \eqref{eq:multi} offers a natural analogy to the first-order logic where each relation is treated as a predicate operator over input arguments (subject and object in our case).  Clearly, the linear transformation defined by a matrix, a.k.a. a linear map, is a richer operator than the additive transformation defined by a vector.
Multiplicative models are also found to substantially outperform additive models
empirically \cite{nickel2011three, DBLP:journals/corr/YangYHGD14a}.

\subsection{Normal Transformations}
Instead of allowing arbitrary linear maps to be used for representing relations, a particular family of matrices has been studied for ``well-behaved'' linear maps.  This family is named as the \textit{normal matrices}.
\begin{definition}[Normal Matrix]
    A real matrix $A$ is normal if and only if $A^\top A = A A^\top$.
\end{definition}

Normal matrices have nice theoretical properties which are often desirable form relational modeling, e.g., they are unitarily diagonalizable and hence can be conveniently analyzed by the spectral theorem \cite{dunford1971linear}.
Representative members of the normal family include:
\begin{itemize}
    \item Symmetric Matrices for which $W_r W_r^\top = W_r^\top W_r= W_r^2 $.  These includes all diagonal matrices and positive semi-definite matrices,
        and the symmetry implies $\phi(s, r, o) = \phi(o, r, s)$.
        They are suitable for modeling symmetric relations such as $is\_identical$.
    \item Skew-/Anti-symmetric Matrices for which $W_rW_r^\top = W_r^\top W_r = -W_r^2$, which implies $\phi(s, r, o) = -\phi(o, r, s)$.
        These matrices are suitable for modeling asymmetric relations
        such as $is\_parent\_of$.
    \item Rotation Matrices for which $W_rW_r^\top  = W_r^\top W_r = I_m$,
        which suggests that the relation $r$ is invertible as $W_r^{-1}$ always exists.
        Rotation matrices are suitable for modeling 1-to-1 relationships (bijections).
    \item Circulant Matrices \cite{gray2006toeplitz},
        which have been implicitly used in recent work on holographic representations \cite{DBLP:conf/aaai/NickelRP16}.
        These matrices are usually related to the learning of latent representations in the Fourier domain
        (see \S \ref{sec:unifiedview} for more details).
\end{itemize}

In the remaining parts of this paper, we denote all the real normal matrices in $\mathbb{R}^{m \times m}$ as $\mathcal{N}_m(\mathbb{R})$.



\section{Proposed Analogical Inference Framework}
\label{sec:analogies}
Analogical reasoning is known to play a central role in human induction about knowledge \cite{gentner1983structure, minsky1988society, holyoak1996mental, hofstadter2001analogy}. Here we provide a mathematical formulation of the analogical structures of interest in multi-relational embedding in a latent semantic space, to support algorithmic inference about the embeddings of entities and relations in a knowledge graph. 

\subsection{Analogical Structures}
Consider the famous example in the word embedding literature \cite{mikolov2013distributed, pennington2014glove},
for the following entities and relations among them: 
\begin{quote}
``$man$ is to $king$ as $woman$ is to $queen$''
\end{quote}
In an abstract notion we denote the entities by $a$ (as $man$) , $b$ (as $king$), $c$ (as $woman$) and $d$ (as $queen$), and the relations by  $r$ (as $crown$)  and $r'$ (as $male \mapsto female$), respectively. These give us the subject-relation-object triplets as follows:
\begin{align}
    a \stackrel{r}{\rightarrow} b, \quad c \stackrel{r}{\rightarrow} d, \quad
    a \stackrel{r'}{\rightarrow} c, \quad b \stackrel{r'}{\rightarrow} d
    \label{eq:abcd}
\end{align}
For multi-relational embeddings,
$r$ and $r'$ are members of $\mathcal{R}$ and are modeled as linear maps in our case.

The relational maps in \eqref{eq:abcd} can be visualized using a commutative diagram \cite{adamek2004abstract, brown2006category} from the Category Theory,
as shown in Figure \ref{fig:abcd},
where each node denotes an entity and each edge denotes a linear map that transforms one entity to the other.
We also refer to such a diagram as a ``parallelogram'' to highlight its particular \emph{algebraic structure}\footnote{Notice that this is different from parallelograms in the geometric sense because each edge here is a linear map instead of the difference between two nodes in the vector space.}.
\begin{figure}[h]
    \centering
    \begin{tikzpicture}
      back line/.style={dashed},
      cross line/.style={preaction={draw=white, -,line width=6pt}}]
      \node (A) {a};
      \node (B) [right of=A, node distance=2cm] {b};
      \node (C) [below of=B, right of=B, node distance=1cm] {d};
      \node (D) [left of=C, node distance=2cm] {c};
      \draw[->] (A) to node [above, sloped] {$r$} (B);
      \draw[->] (B) to node [above] {$r'$} (C);
      \draw[->] (D) to node [above, sloped] {$r$} (C);
      \draw[->] (A) to node [above] {$r'$} (D);
    \end{tikzpicture}
    \caption{Parallelogram diagram for the analogy of ``$a$ is to $b$ as $c$ is to $d$'', where each edge denotes a linear map.}
    \label{fig:abcd}
\end{figure}
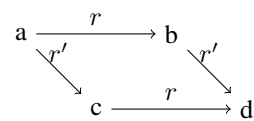

The parallelogram in Figure \ref{fig:abcd} represents a very basic analogical structure
which could be informative for the inference about unknown facts (triplets).
To get a sense about why analogies would help in the inference about unobserved facts, we notice that
for entities $a, b, c, d$ which form an analogical structure in our example,
the parallelogram structure is fully determined by symmetry. This means that if we know $a \stackrel{r}{\rightarrow} b$ and $a \stackrel{r'}{\rightarrow} c$, then we can induce the remaining triplets of $c \stackrel{r}{\rightarrow} d$  and $b \stackrel{r'}{\rightarrow} d$.  In other words, 
understanding the relation between $man$ and $king$ helps us to fill up the  unknown relation between $woman$ and $queen$.


Analogical structures are not limited to parallelograms, of course, 
though parallelograms often serve as the building blocks for more complex analogical structures.
As an example, in Figure \ref{fig:solar} of \S \ref{sec:introduction} we show a compound analogical structure in the form of a triangular prism, for mirroring the correspondent entities/relations between 
the atom system and the solar system.
Formally define the desirable analogical structures in a  computationally tractable objective for optimization is the key for solving our problem, which we will introduce next.  

\subsection{Commutative Constraint for Linear Maps}
Although it is tempting to explore all potentially interesting parallelograms in the modeling of analogical structure, it is computationally intractable to examine the entire powerset of entities as the candidate space of analogical structures.
A more reasonable strategy is to identify some desirable properties of the analogical structures we want to model,  and use those properties as constraints for reducing the candidate space. 

An desirable property of the linear maps we want is 
that all the directed paths with the same starting node and end node form the \textit{compositional equivalence}.
Denoting by ``$\circ$'' the composition operator between two relations,
the parallelogram in Figure \ref{fig:abcd} contains two  equivalent compositions as:
\begin{equation}
    r \circ r' = r' \circ r
    \label{eq:rcommute}
\end{equation}
which means that $a$ is connected to $d$ via either path.
We call this the \textit{commutativity} property of the linear maps, which is a
necessary condition for forming commutative parallelograms and therefore the corresponding analogical structures.
Yet another example is given by Figure \ref{fig:solar},
where $sun$ can traverse to $charge$ along multiple alternative paths of length three,
implying the commutativity of relations $surrounded\_by$, $made\_of$, $scale\_down$.

The composition of two relations (linear maps) is naturally implemented via matrix multiplication \cite{DBLP:journals/corr/YangYHGD14a, guu2015traversing},
hence equation \eqref{eq:rcommute} indicates
\begin{equation}
    W_{r\circ r'} = W_r W_{r'} = W_{r'} W_r
    \label{eq:wcommute}
\end{equation}
One may further require the commutative constraint \eqref{eq:wcommute} to be satisfied for any pair of relations in $\mathcal{R}$ because they may be simultaneously present in the same commutative parallelogram for certain subsets of entities.
In this case,
we say the relations in $\mathcal{R}$ form a commuting family.


It is worth mentioning that $\mathcal{N}_m(\mathbb{R})$ is not closed under matrix multiplication.
As the result,
the composition rule in eq.\ \eqref{eq:wcommute} may not always yield a legal new relation---$W_{r \circ r'}$ may no longer be a normal matrix.
However,
any commuting family in $\mathcal{N}_m(\mathbb{R})$ is indeed closed under multiplication.
This explains the necessity of having a commuting family of relations from an alternative perspective.

\subsection{The Optimization Objective}
The generic goal for multi-relational embedding
is to find entity and relation representations 
such that positive triples labeled as $y=+1$ receive higher score than the negative triples labeled as $y=-1$.
This can be formulated as
\begin{align}
    \min_{v, W} \enskip& \mathbb{E}_{s,r,o,y \sim \mathcal{D}}\enskip \ell\left(\phi_{v, W}( s, r, o ), y\right)
\end{align}
where $\phi_{v, W}(s, r, o) = v_s^\top W_r v_o$ is our score function based on the embeddings,
$\ell$ is our loss function,
and $\mathcal{D}$ is the data distribution constructed based on the training set $\mathcal{K}$.

To impose analogical structures among the representations,
we in addition require the linear maps associated with relations
to form a commuting family of normal matrices.
This gives us the objective function for ANALOGY:
\begin{align}
    \min_{v, W} \enskip &\mathbb{E}_{s,r,o,y \sim \mathcal{D}}\enskip \ell\left(\phi_{v, W}( s, r, o ), y\right) \label{obj-original} \\
    \text{s.t.} \enskip& W_r W_r^\top = W_r^\top W_r \quad \forall r \in \mathcal{R} \label{obj-original-normal} \\
    \enskip & W_rW_{r'} = W_{r'}W_r \quad \forall r, r' \in \mathcal{R} \label{eq:opt-commutative}
\end{align}
where constraints \eqref{obj-original-normal} and \eqref{eq:opt-commutative} are corresponding to the normality and commutativity requirements,
respectively.
Such a constrained optimization may appear to be computationally expensive at the first glance.
In \S \ref{sec:algorithm}, however, we will recast it as a simple lightweight problem
for which each SGD update can be carried out efficiently in $O(m)$ time.

\section{Efficient Inference Algorithm}
\label{sec:algorithm}

The constrained optimization \eqref{obj-original}
is computationally challenging due to the large number of model parameters in tensor $W$,
the matrix normality constraints,
and the quadratic number of pairwise commutative constraints in \eqref{eq:opt-commutative}.

Interestingly,
by exploiting the special properties of commuting normal matrices,
we will show in Corollary \ref{coro:alternative} that ANALOGY can be alternatively solved
via an another formulation of substantially lower complexity.
Our findings are based on the following lemma and theorem:

\begin{lemma} \cite{wilkinson1965algebraic}
    For any real normal matrix $A$,
    there exists a real orthogonal matrix $Q$ and a block-diagonal matrix $B$ such that $A = QBQ^\top$, where each diagonal block of $B$ is either (1) A real scalar, or (2) A 2-dimensional real matrix in the form of
            $
            \begin{bmatrix}
                x & -y \\
                y & x 
            \end{bmatrix}
            $, where both $x$, $y$ are real scalars.

    \label{lem:diagonal}
\end{lemma}
The lemma suggests any real normal matrix can be block-diagonalized
into an almost-diagonal canonical form.

\begin{theorem}
    [Proof given in the supplementary material]
    If a set of real normal matrices $A_1, A_2, ... $ form a commuting family,
    namely
    $
    A_i A_j = A_j A_i\ \forall i, j$,
    then they can be block-diagonalized by the same real orthogonal basis $Q$.
    \label{thm:commutativity}
\end{theorem}
The theorem above implies that the set of dense relational matrices $\{W_r\}_{r \in \mathcal{R}}$,
if mutually commutative,
can always be \emph{simultaneously block-diagonalized}
into another set of sparse almost-diagonal matrices $\{B_r\}_{r \in \mathcal{R}}$.
\begin{corollary}[Alternative formulation for ANALOGY]
    For any given solution $(v^*, W^*)$ of optimization \eqref{obj-original},
    there always exists an alternative set of embeddings $(u^*, B^*)$
    such that $\phi_{v^*, W^*}(s, r, o) \equiv \phi_{u^*, B^*}(s, r, o)$,
    $\forall (s, r, o)$,
    and $(u^*, B^*)$ is given by the solution of:
    \begin{align}
        \min_{u, B} &\enskip \mathbb{E}_{s,r,o,y \sim \mathcal{D}}\enskip \ell\left(\phi_{u, B}(s, r, o), y\right) \label{obj-alternative} \\
        &\enskip B_r \in \mathcal{B}^n_m \quad \forall r \in \mathcal{R} \label{eq:constraint-B}
    \end{align}
    where $\mathcal{B}^n_m$ denotes all $m \times m$ almost-diagonal matrices 
    in Lemma \ref{lem:diagonal} with $n < m$ real scalars on the diagonal.
    \label{coro:alternative}
\end{corollary}

\begin{proof}[proof sketch]
    With the commutative constraints,
    there must exist some orthogonal matrix $Q$,
    such that
    $ W_r = QB_rQ^\top $, $B_r \in \mathcal{B}^n_m$, $\forall r \in \mathcal{R}$.
    We can plug-in these expressions into optimization \eqref{obj-original}
    and let $u = vQ$,
    obtaining
\begin{align}
    \phi_{v, W}(s, r, o) =& v_s^\top W_r v_o
    =v_s^\top Q B_r Q^\top v_o \\
    =&u_s^\top B_r u_o = \phi_{u, B}(s, r, o)
\end{align}
In addition, it is not hard to verify that constraints \eqref{obj-original-normal} and \eqref{eq:opt-commutative}
are automatically satisfied by exploiting the facts that $Q$ is orthogonal and $\mathcal{B}_m^n$ is a commutative normal family.
\end{proof}

Constraints \eqref{eq:constraint-B} in the alternative optimization problem
can be handled by simply binding together the coefficients within each of those $2\times2$ blocks in $B_r$.
Note that each $B_r$ consists of only $m$ free parameters,
allowing the gradient w.r.t.\ any given triple to be efficiently evaluated in $O(m)$.

\section{Unified View of Representative Methods}
\label{sec:unifiedview}
In the following we provide a unified view of several embedding models
\cite{DBLP:journals/corr/YangYHGD14a, DBLP:conf/icml/TrouillonWRGB16, DBLP:conf/aaai/NickelRP16},
by showing that they are restricted versions under our framework,
hence are implicitly imposing analogical properties.
This explains their strong empirical performance as compared to other baselines (\S \ref{sec:experiments}).

\subsection{DistMult}
DistMult \cite{DBLP:journals/corr/YangYHGD14a} embeds both entities and relations as
vectors, and defines the score function as
\begin{align}
    \phi(s, r, o) &= \langle v_s, v_r, v_o \rangle \\
    \text{where} \enskip v_s, v_r, v_o &\in \mathbb{R}^m, \forall s, r, o
    \label{eq:score-distmult}
\end{align}
where $\langle \cdot, \cdot, \cdot \rangle$ denotes the generalized inner product.

\begin{proposition}
    DistMult embeddings can be fully recovered by ANALOGY embeddings when $n = m$.
\end{proposition}
\begin{proof}
    This is trivial to verify as the score function \eqref{eq:score-distmult} can be rewritten as $\phi(s, r, o) = v_s^\top B_r v_o$
    where $B_r$ is a diagonal matrix given by $B_r = \mathrm{diag}(v_r)$.
\end{proof}
Entity analogies are encouraged in DistMult as the diagonal matrices $\mathrm{diag}(v_r)$'s are both normal and mutually commutative.
However, DistMult is restricted to model symmetric relations only, since $\phi(s, r, o) \equiv \phi(o, r, s)$.

\subsection{Complex Embeddings (ComplEx)}
ComplEx \cite{DBLP:conf/icml/TrouillonWRGB16} extends the embeddings to the complex domain $\mathbb{C}$,
which defines
\begin{align}
    \phi(s, r, o) &= \Re\left( \langle v_s, v_r, \overline{v_o} \rangle \right) \label{eq:complex} \\
    \text{where} \enskip v_s, v_r, v_o &\in \mathbb{C}^m, \forall s, r, o
\end{align}
where $\overline{x}$ denotes the complex conjugate of $x$.

\begin{proposition}
    ComplEx embeddings of embedding size $m$ can be fully recovered by ANALOGY embeddings
    of embedding size $2m$ when $n = 0$.
\end{proposition}

\begin{proof}
    Let $\Re(x)$ and $\Im(x)$ be the real and imaginary parts of any complex vector $x$. We recast $\phi$ in \eqref{eq:complex} as
\begin{align}
    \phi(r, s, o) 
    = &+ \big\langle \Re(v_r), \Re(v_s), \Re(v_o) \big\rangle \\
    &+ \big\langle \Re(v_r), \Im(v_s), \Im(v_o) \big\rangle \\
    &+ \big\langle \Im(v_r), \Re(v_s), \Im(v_o) \big\rangle \\
    &- \big\langle \Im(v_r), \Im(v_s), \Re(v_o) \big\rangle
    = {v_s'}^\top B_r v_o'
\end{align}
The last equality is obtained via a change of variables:
For any complex entity embedding $v\in \mathbb{C}^{m}$,
we define a new real embedding $v' \in \mathbb{R}^{2m}$ such that
\begin{align}
    \begin{cases}
        (v')_{2k} &= \Re(v)_k \\
        (v')_{2k-1} &= \Im(v)_k
    \end{cases} \quad \forall k = 1,2,\dots m
\end{align}
The corresponding $B_r$ is a block-diagonal matrix in $\mathcal{B}^0_{2m}$ with its $k$-th block given by
$
    \begin{bmatrix}
        \Re(v_r)_k & - \Im(v_r)_k \\
        \Im(v_r)_k & \Re(v_r)_k \\
    \end{bmatrix}
    $.
\end{proof}


\subsection{Holographic Embeddings (HolE)}
HolE \cite{DBLP:conf/aaai/NickelRP16} defines the score function as
\begin{align}
    \phi(s, r, o) &= \langle v_r, v_s \ast v_o \rangle \\
    \text{where} \enskip v_s, v_r, v_o &\in \mathbb{R}^m, \forall s, r, o
    \label{eq:hole}
\end{align}
where the association of $s$ and $o$ is implemented via circular correlation denoted by $\ast$.
This formulation is motivated by the holographic reduced representation \cite{plate2003holographic}.

To relate HolE with ANALOGY,
we rewrite \eqref{eq:hole} in a bilinear form with a circulant matrix $C(v_r)$ in the middle 
\begin{align}
    \phi(r, s, o)
    &= v_s^\top 
    C(v_r)
    v_o
\end{align}
where entries of a circulant matrix are defined as
\begin{equation}
    C(x) = 
    \begin{bmatrix}
        x_1 & x_m & \cdots & x_3 & x_2 \\
        x_2 & x_1 & x_m & & x_3 \\
        \vdots & x_2 & x_1 & \ddots & \vdots \\
        x_{m-1} & & \ddots & \ddots & x_m \\
        x_{m} & x_{m-1} & \cdots & x_2 & x_1 \\
    \end{bmatrix}
\end{equation}
It is not hard to verify that circulant matrices are normal
and commute \cite{gray2006toeplitz},
hence entity analogies are encouraged in HolE,
for which optimization \eqref{obj-original} reduces to an unconstrained problem
as equalities \eqref{obj-original-normal} and \eqref{eq:opt-commutative} are automatically satisfied when all $W_r$'s are circulant.

The next proposition
further reveals that HolE is equivalent to ComplEx with minor relaxation.
\begin{proposition}
    HolE embeddings can be equivalently obtained using the following score function 
    \begin{align}
        \phi(s, r, o) &= \Re\left( \langle v_s, v_r, \overline{v_o} \rangle \right) \\
        \text{where} \enskip v_s, v_r, v_o &\in \fft(\mathbb{R}^m), \forall s, r, o
    \end{align}
    where $\fft(\mathbb{R}^m)$ denotes the image of $\mathbb{R}^m$ in $\mathbb{C}^m$ through the Discrete Fourier Transform (DFT).
    In particular, the above reduces to ComplEx by relaxing $\fft(\mathbb{R}^m)$ to $\mathbb{C}^m$.
\end{proposition}

\begin{proof}
    Let $\fft$ be the DFT operator
    defined by $\fft(x) = Fx$ where $F \in \mathbb{C}^{m \times m}$ is called the Fourier basis of DFT.
    A well-known property for circulant matrices is that any $C(x)$ can always be diagonalized by $F$,
    and its eigenvalues are given by $Fx$ \cite{gray2006toeplitz}.

    Hence the score function can be further recast as
    \begin{align}
        \phi(r, s, o)
        &= v_s^\top F^{-1} \diag(F v_r) F v_o \\
        &= \frac{1}{m} \overline{(F v_s)}^\top \diag(F v_r) (F v_o) \\
        &= \frac{1}{m} \langle \overline{\fft(v_s)}, \fft(v_r), \fft(v_o) \rangle \\
        &= \Re \left[ \frac{1}{m} \langle \overline{\fft(v_s)}, \fft(v_r), \fft(v_o) \rangle \right]
    \end{align}
    Let $v_s' = \overline{\fft(v_s)}, v_o' = \overline{\fft(v_o)}$ and $v_r' = \frac{1}{m} \fft(v_r)$,
    we obtain exactly the same score function as used in ComplEx
    \begin{equation}
        \phi(s, r, o) = \Re \left(\langle v_s', v_r', \overline{v_o'} \rangle\right)
        \label{eq:complex-prime}
    \end{equation}
    \eqref{eq:complex-prime} is equivalent to \eqref{eq:complex}
    apart from an additional constraint that $v_s', v_r', v_o'$ are
    the image of $\mathbb{R}$ in the Fourier domain.
\end{proof}

\section{Experiments}
\label{sec:experiments}

\subsection{Datasets}
We evaluate ANALOGY and the baselines over two benchmark datasets for multi-relational embedding released by previous work \cite{bordes2013translating},
namely a subset of Freebase (FB15K) for generic facts and WordNet (WN18) for lexical relationships between words.

The dataset statistics are summarized in Table \ref{tab:datasets}.

\begin{table}[h]
    \centering
    \begin{tabularx}{\linewidth}{cccccX}
        \toprule
        Dataset & $|\mathcal{E}|$ & $|\mathcal{R}|$ & \#train & \#valid & \#test \\
        \midrule
        FB15K & 14,951 & 1,345 & 483,142 & 50,000 & 59,071 \\
        WN18 & 40,943 & 18 & 141,442 & 5,000 & 5,000 \\
        \bottomrule
    \end{tabularx}
    \caption{Dataset statistics for FB15K and WN18.}
    \label{tab:datasets}
\end{table}

\subsection{Baselines}
\label{sec:baselines}
We compare the performance of ANALOGY against a variety types of multi-relational embedding models
developed in recent years. Those models can be categorized as: 
\begin{itemize}
    \item Translation-based models
        where relations are modeled as translation operators in the embedding space,
        including TransE \cite{bordes2013translating} and its variants TransH \cite{wang2014knowledge}, TransR \cite{lin2015learning}, TransD \cite{ji2015knowledge},
        STransE \cite{nguyen2016stranse} and RTransE \cite{garcia2015composing}.
    \item Multi-relational latent factor models including LFM \cite{jenatton2012latent}
        and RESCAL \cite{nickel2011three} based collective matrix factorization.
    \item Models involving neural network components such as neural tensor networks \cite{socher2013reasoning} and PTransE-RNN \cite{lin2015learning},
        where RNN stands for recurrent neural networks.
    \item Pathwise models including three different variants of PTransE \cite{lin2015modeling} which extend TransE by explicitly taking into account
        indirect connections (relational paths) between entities.
    \item Models subsumed under our proposed framework (\S \ref{sec:unifiedview}),
        including DistMult \cite{DBLP:journals/corr/YangYHGD14a} based simple multiplicative interactions, ComplEx \cite{DBLP:conf/icml/TrouillonWRGB16} using complex coefficients
        and HolE \cite{DBLP:conf/aaai/NickelRP16} based on holographic representations.
        Those models are implicitly leveraging analogical structures per our previous analysis.
    \item Models enhanced by external side information.
        We use Node+LinkFeat (NLF) \cite{toutanova2015observed} as a representative example,
        which leverages textual mentions derived from the ClueWeb corpus.
\end{itemize}


\subsection{Evaluation Metrics}

Following the literature of multi-relational embedding, we use the conventional metrics of Hits@k and Mean Reciprocal Rank (MRR) which evaluate each system-produced ranked list for each test instance and average the scores over all ranked lists for the entire test set of instances.

The two metrics would be flawed for the \textit{negative instances} created in the test phase as a ranked list may contain some positive instances in the training and validation sets \cite{bordes2013translating}.
A recommended remedy, which we followed, is to remove all training- and validation-set triples from all ranked lists during testing. 
We use ``filt.''\ and ``raw'' to indicate the evaluation metrics with or without filtering, respectively.

In the first set of our experiments,
we used on Hits@k with k=10,
which has been reported for most methods in the literature.  We also provide additional results of ANALOGY and a subset of representative baseline methods using MRR, Hits@1 and Hits@3, to enable the comparison with the methods whose published results are in those metrics.


\subsection{Implementation Details}
\subsubsection{Loss Function}
We use the logistic loss for ANALOGY throughout all experiments, namely
$\ell(\phi(s, r, o), y) = - \log \sigma (y \phi(s, r, o))$,
where $\sigma$ is the sigmoid activation function.
We empirically found this simple loss function to perform reasonably well as compared
to more sophisticated ranking loss functions.

\subsubsection{Asynchronous AdaGrad}
Our C++ implementation\footnote{Code available at \href{https://github.com/quark0/ANALOGY}{https://github.com/quark0/ANALOGY}.} runs over a CPU, as ANALOGY only requires lightweight linear algebra routines.
We use asynchronous stochastic gradient descent (SGD) for optimization,
where the gradients with respect to different mini-batches are simultaneously evaluated in multiple threads,
and the gradient updates for the shared model parameters are carried out without synchronization.
Asynchronous SGD is highly efficient,
and causes little performance drop when parameters associated with different mini-batches are mutually disjoint with a high probability
\cite{recht2011hogwild}.
We adapt the learning rate based on historical gradients
using AdaGrad \cite{duchi2011adaptive}.

\subsubsection{Creation of Negative Samples}
Since only valid triples (positive instances) are explicitly given in the training set,
invalid triples (negative instances) need to be artificially created.
Specifically,
for every positive example $(s, r, o)$,
we generate three negative instances $(s', r, o)$, $(s, r', o)$, $(s, r, o')$ by corrupting $s$, $r$, $o$ with random entities/relations $s' \in \mathcal{E}$, $r' \in \mathcal{R}$, $o' \in \mathcal{E}$.
The union of all positive and negative instances
defines our data distribution $\mathcal{D}$ for SGD updates.

\subsubsection{Model Selection}
We conducted a grid search to find the hyperparameters of ANALOGY which maximize the filtered MRR on the validation set,
by enumerating all combinations of the embedding size $m \in \{100, 150, 200\}$,
$\ell_2$ weight decay factor $\lambda \in \{10^{-1}, 10^{-2}, 10^{-3}\}$ of model coefficients $v$ and $W$,
and the ratio of negative over positive samples $\alpha \in \{3, 6\}$.
The resulting hyperparameters for the WN18 dataset are $m = 200, \lambda = 10^{-2}, \alpha = 3$,
and those for the FB15K dataset are $m = 200, \lambda = 10^{-3}, \alpha = 6$.
The number of scalars on the diagonal of each $B_r$ is always set to be $\frac{m}{2}$.
We set the initial learning rate to be $0.1$ for both datasets and adjust it using AdaGrad during optimization.
All models are trained for 500 epochs.

\subsection{Results}

\begin{table}[t]
    \centering
    \caption{Hits@10 (filt.) of all models on WN18 and FB15K categories into three groups: (i) 19 baselines without modeling analogies; (ii) 3 baselines and our proposed ANALOGY which implicitly or explicitly enforce analogical properties over the induced embeddings (see \S \ref{sec:unifiedview}); (iii) One baseline relying on large external data resources in addition to the provided training set.}
    \begin{tabular}{@{}l|cc@{}}
        \toprule
        Models                & WN18          & FB15K       \\ \midrule
        Unstructured \cite{bordes2013translating}     & 38.2          & 6.3         \\
        RESCAL \cite{nickel2011three}               & 52.8          & 44.1        \\
        NTN \cite{socher2013reasoning}              & 66.1          & 41.4        \\
        SME \cite{bordes2012joint}                  & 74.1          & 41.3        \\
        SE \cite{bordes2011learning}                    & 80.5          & 39.8        \\
        LFM \cite{jenatton2012latent}                  & 81.6          & 33.1        \\
        TransH \cite{wang2014knowledge}               & 86.7          & 64.4        \\
        TransE \cite{bordes2013translating}               & 89.2          & 47.1        \\
        TransR \cite{lin2015learning}               & 92.0            & 68.7        \\
        TKRL  \cite{xie2016representation} & -- & 73.4 \\
        RTransE \cite{garcia2015composing}              & --            & 76.2        \\
        TransD \cite{ji2015knowledge}               & 92.2          & 77.3        \\
        CTransR \cite{lin2015learning}              & 92.3          & 70.2        \\
        KG2E  \cite{he2015learning} & 93.2 & 74.0 \\
        STransE \cite{nguyen2016stranse}              & 93.4          & 79.7        \\
        DistMult \cite{DBLP:journals/corr/YangYHGD14a}    & 93.6          & 82.4        \\
        TransSparse \cite{DBLP:conf/aaai/JiLH016}              & 93.9          & 78.3        \\
        PTransE-MUL \cite{lin2015modeling}	& --            & 77.7        \\
        PTransE-RNN \cite{lin2015modeling}  & --            & 82.2        \\
        PTransE-ADD \cite{lin2015modeling}  & --            & 84.6        \\
        \pbox{4cm}{NLF (with external corpus) \\ \cite{toutanova2015observed}}         & \textit{94.3} & \textit{87.0} \\
        ComplEx \cite{DBLP:conf/icml/TrouillonWRGB16}              & \textbf{94.7}          & 84.0        \\
        HolE   \cite{DBLP:conf/aaai/NickelRP16}               & \textbf{94.9}          & 73.9        \\
        \midrule
        Our ANALOGY               &     \textbf{94.7}      & \textbf{85.4}   \\ 
        \bottomrule
    \end{tabular}
    \label{tab:results-main}
\end{table}

\begin{table*}[t]
    \centering
    \caption{MRR and Hits@\{1,3\} of a subset of representative models on WN18 and FB15K. The performance scores of TransE and REACAL are cf. the results published in \cite{DBLP:conf/icml/TrouillonWRGB16} and \cite{DBLP:conf/aaai/NickelRP16}, respectively.}
    \begin{tabular}{@{}l|cccc|cccc@{}}
        \toprule
        & \multicolumn{4}{c}{WN18}                                      & \multicolumn{4}{c}{FB15}                                      \\ \midrule
        Models        & \pbox{1.0cm}{MRR \\ (filt.)}          & \pbox{1.0cm}{MRR \\ (raw)}      & \pbox{1.0cm}{Hits@1 \\ (filt.)}       & \pbox{1.0cm}{Hits@3 \\ (filt.)}       & \pbox{1.0cm}{MRR \\ (filt.)}          & \pbox{1.0cm}{MRR \\ (raw)}      & \pbox{1.0cm}{Hits@1 \\ (filt.)}        & \pbox{1.0cm}{Hits@3 \\ (filt.)}     \\ \midrule
        RESCAL  \cite{nickel2011three}      & 89.0          & 60.3          & 84.2          & 90.4          & 35.4          & 18.9          & 23.5          & 40.9          \\
        TransE   \cite{bordes2013translating}     & 45.4          & 33.5          & 8.9           & 82.3          & 38.0          & 22.1          & 23.1          & 47.2          \\ 
        DistMult  \cite{DBLP:journals/corr/YangYHGD14a}    & 82.2          & 53.2          & 72.8          & 91.4          & 65.4          & 24.2          & 54.6          & 73.3          \\
        HolE  \cite{DBLP:conf/aaai/NickelRP16}        & 93.8          & 61.6          & \textbf{93.0}          & \textbf{94.5} & 52.4          & 23.2          & 40.2          & 61.3          \\ 
        ComplEx    \cite{DBLP:conf/icml/TrouillonWRGB16}   & 94.1          & 58.7          & \textbf{93.6}          & \textbf{94.5} & 69.2          & 24.2          & 59.9          & 75.9          \\
        \midrule
        Our ANALOGY & \textbf{94.2} & \textbf{65.7} & \textbf{93.9}          & \textbf{94.4}          & \textbf{72.5} & \textbf{25.3} & \textbf{64.6} & \textbf{78.5} \\ \bottomrule
    \end{tabular}
    \label{tab:results-representative}
\end{table*}

Table \ref{tab:results-main} compares the Hits@10 score of ANALOGY with that of 23 competing methods using the published scores for these methods in the literature on the WN18 and FB15K datasets.  For the methods not having both scores, the missing slots are indicated by ``--''. 
The best score on each dataset is marked in the bold face; if the differences among the top second or third scores are not statistically significant from the top one, then these scores are also bold faced.  We used one-sample proportion test \cite{yang1999re} at the 5\% p-value level for testing the statistical significances\footnote{Notice that proportion tests only apply to performance scores as proportions, including Hits@k, but are not applicable to non-proportional scores such as MRR. Hence we only conducted the proportion tests on the Hits@k scores.}. 

Table \ref{tab:results-representative} compares the methods (including ours) whose results in additional metrics are available.
The usage of the bold faces is the same as those in Table \ref{tab:results-main}.

In both tables, ANALOGY performs either the best or the 2nd best which is in the equivalent class with the best score in each case according statistical significance test. Specifically, on the harder FB15K dataset in Table \ref{tab:results-main}, which has a very large number of relations, our model outperforms all baseline methods. These results provide a good evidence for the effective modeling of analogical structures in our approach.
We are pleased to see in Table \ref{tab:results-representative} that ANALOGY outperforms DistMult, ComplEx and HolE in all the metrics, as the latter three can be viewed as more constrained versions of our method (as discussed in (\S \ref{sec:unifiedview})). Furthermore, our assertion on HolE for being a special case of ComplEx (\S \ref{sec:unifiedview}) is justified in the same table by the fact that the performance of HolE is dominated by ComplEx. 

\begin{figure}
\centering
	\includegraphics[width=0.475\linewidth]{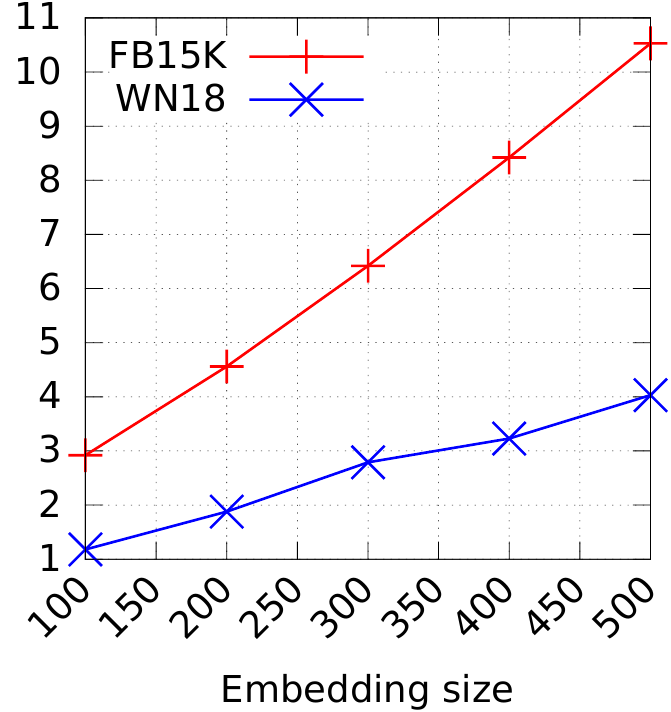}
    \includegraphics[width=0.475\linewidth]{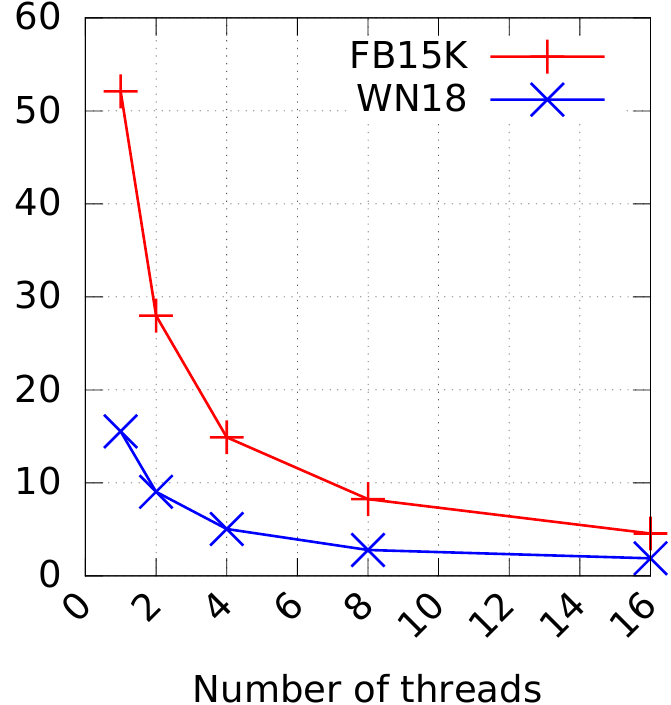}
    \caption{CPU run time per epoch (secs) of ANALOGY. The figure on the left shows the run time over increasing embedding sizes with 16 CPU threads; Figure on the right shows the run time over increasing number of CPU threads with embedding size 200.}
    \label{fig:speed}
\end{figure}

In Figure \ref{fig:speed} we show the empirical scalability of ANALOGY,
which not only completes one epoch in a few seconds on both datasets, 
but also scales linearly in the size of the embedding problem.
As compared to single-threaded AdaGrad,
our asynchronous AdaGrad over 16 CPU threads
offers 11.4x and 8.3x speedup on FB15K and WN18, respectively,
on a single commercial desktop.

\section{Conclusion}
\label{sec:conclusion}
We presented a novel framework for explicitly modeling analogical structures in multi-relational embedding, along with a differentiable objective function and a linear-time inference algorithm for large-scale embedding of knowledge graphs.
The proposed approach obtains the state-of-the-art results on two popular benchmark datasets, outperforming a large number of strong baselines in most cases. 

Although we only focused on the multi-relational inference for knowledge-base embedding, we believe that analogical structures exist in many other machine learning problems beyond the scope of this paper. We hope this work shed light on a broad range of important problems where scalable inference for analogical analysis would make an impact,
such as machine translation and image captioning (both problems require modeling cross-domain analogies).  We leave these interesting topics as our future work.

\section*{Acknowledgments}
We thank the reviewers for their helpful comments. This work is supported in part by the National Science Foundation (NSF) under grant IIS-1546329.

\bibliography{reference}

\begin{thebibliography}{49}
\providecommand{\natexlab}[1]{#1}
\providecommand{\url}[1]{\texttt{#1}}
\expandafter\ifx\csname urlstyle\endcsname\relax
  \providecommand{\doi}[1]{doi: #1}\else
  \providecommand{\doi}{doi: \begingroup \urlstyle{rm}\Url}\fi

\bibitem[Ad{\'a}mek et~al.(2004)Ad{\'a}mek, Herrlich, and
  Strecker]{adamek2004abstract}
Ad{\'a}mek, Ji{\v{r}}{\'\i}, Herrlich, Horst, and Strecker, George~E.
\newblock Abstract and concrete categories. the joy of cats.
\newblock 2004.

\bibitem[Auer et~al.(2007)Auer, Bizer, Kobilarov, Lehmann, Cyganiak, and
  Ives]{auer2007dbpedia}
Auer, S{\"o}ren, Bizer, Christian, Kobilarov, Georgi, Lehmann, Jens, Cyganiak,
  Richard, and Ives, Zachary.
\newblock Dbpedia: A nucleus for a web of open data.
\newblock In \emph{The semantic web}, pp.\  722--735. Springer, 2007.

\bibitem[Bollacker et~al.(2008)Bollacker, Evans, Paritosh, Sturge, and
  Taylor]{bollacker2008freebase}
Bollacker, Kurt, Evans, Colin, Paritosh, Praveen, Sturge, Tim, and Taylor,
  Jamie.
\newblock Freebase: a collaboratively created graph database for structuring
  human knowledge.
\newblock In \emph{Proceedings of the 2008 ACM SIGMOD international conference
  on Management of data}, pp.\  1247--1250. AcM, 2008.

\bibitem[Bordes et~al.(2011)Bordes, Weston, Collobert, and
  Bengio]{bordes2011learning}
Bordes, Antoine, Weston, Jason, Collobert, Ronan, and Bengio, Yoshua.
\newblock Learning structured embeddings of knowledge bases.
\newblock In \emph{Conference on artificial intelligence}, number
  EPFL-CONF-192344, 2011.

\bibitem[Bordes et~al.(2012)Bordes, Glorot, Weston, and
  Bengio]{bordes2012joint}
Bordes, Antoine, Glorot, Xavier, Weston, Jason, and Bengio, Yoshua.
\newblock Joint learning of words and meaning representations for open-text
  semantic parsing.
\newblock In \emph{AISTATS}, volume~22, pp.\  127--135, 2012.

\bibitem[Bordes et~al.(2013)Bordes, Usunier, Garcia-Duran, Weston, and
  Yakhnenko]{bordes2013translating}
Bordes, Antoine, Usunier, Nicolas, Garcia-Duran, Alberto, Weston, Jason, and
  Yakhnenko, Oksana.
\newblock Translating embeddings for modeling multi-relational data.
\newblock In \emph{Advances in neural information processing systems}, pp.\
  2787--2795, 2013.

\bibitem[Brown \& Porter(2006)Brown and Porter]{brown2006category}
Brown, Ronald and Porter, Tim.
\newblock Category theory: an abstract setting for analogy and comparison.
\newblock In \emph{What is category theory}, volume~3, pp.\  257--274, 2006.

\bibitem[Chen et~al.(2013)Chen, Socher, Manning, and Ng]{chen2013learning}
Chen, Danqi, Socher, Richard, Manning, Christopher~D, and Ng, Andrew~Y.
\newblock Learning new facts from knowledge bases with neural tensor networks
  and semantic word vectors.
\newblock \emph{arXiv preprint arXiv:1301.3618}, 2013.

\bibitem[Dalton et~al.(2014)Dalton, Dietz, and Allan]{dalton2014entity}
Dalton, Jeffrey, Dietz, Laura, and Allan, James.
\newblock Entity query feature expansion using knowledge base links.
\newblock In \emph{Proceedings of the 37th international ACM SIGIR conference
  on Research \& development in information retrieval}, pp.\  365--374. ACM,
  2014.

\bibitem[Duchi et~al.(2011)Duchi, Hazan, and Singer]{duchi2011adaptive}
Duchi, John, Hazan, Elad, and Singer, Yoram.
\newblock Adaptive subgradient methods for online learning and stochastic
  optimization.
\newblock \emph{Journal of Machine Learning Research}, 12\penalty0
  (Jul):\penalty0 2121--2159, 2011.

\bibitem[Dunford et~al.(1971)Dunford, Schwartz, Bade, and
  Bartle]{dunford1971linear}
Dunford, Nelson, Schwartz, Jacob~T, Bade, William~G, and Bartle, Robert~G.
\newblock \emph{Linear operators}.
\newblock Wiley-interscience New York, 1971.

\bibitem[Falkenhainer et~al.(1989)Falkenhainer, Forbus, and
  Gentner]{falkenhainer1989structure}
Falkenhainer, Brian, Forbus, Kenneth~D, and Gentner, Dedre.
\newblock The structure-mapping engine: Algorithm and examples.
\newblock \emph{Artificial intelligence}, 41\penalty0 (1):\penalty0 1--63,
  1989.

\bibitem[Ferrucci et~al.(2010)Ferrucci, Brown, Chu-Carroll, Fan, Gondek,
  Kalyanpur, Lally, Murdock, Nyberg, Prager, et~al.]{ferrucci2010building}
Ferrucci, David, Brown, Eric, Chu-Carroll, Jennifer, Fan, James, Gondek, David,
  Kalyanpur, Aditya~A, Lally, Adam, Murdock, J~William, Nyberg, Eric, Prager,
  John, et~al.
\newblock Building watson: An overview of the deepqa project.
\newblock \emph{AI magazine}, 31\penalty0 (3):\penalty0 59--79, 2010.

\bibitem[Gabrilovich \& Markovitch(2009)Gabrilovich and
  Markovitch]{gabrilovich2009wikipedia}
Gabrilovich, Evgeniy and Markovitch, Shaul.
\newblock Wikipedia-based semantic interpretation for natural language
  processing.
\newblock \emph{Journal of Artificial Intelligence Research}, 34:\penalty0
  443--498, 2009.

\bibitem[Garcia-Duran et~al.(2015)Garcia-Duran, Bordes, and
  Usunier]{garcia2015composing}
Garcia-Duran, Alberto, Bordes, Antoine, and Usunier, Nicolas.
\newblock \emph{Composing relationships with translations}.
\newblock PhD thesis, CNRS, Heudiasyc, 2015.

\bibitem[Gentner(1983)]{gentner1983structure}
Gentner, Dedre.
\newblock Structure-mapping: A theoretical framework for analogy.
\newblock \emph{Cognitive science}, 7\penalty0 (2):\penalty0 155--170, 1983.

\bibitem[Getoor(2007)]{getoor2007introduction}
Getoor, Lise.
\newblock \emph{Introduction to statistical relational learning}.
\newblock MIT press, 2007.

\bibitem[Gray et~al.(2006)]{gray2006toeplitz}
Gray, Robert~M et~al.
\newblock Toeplitz and circulant matrices: A review.
\newblock \emph{Foundations and Trends{\textregistered} in Communications and
  Information Theory}, 2\penalty0 (3):\penalty0 155--239, 2006.

\bibitem[Guu et~al.(2015)Guu, Miller, and Liang]{guu2015traversing}
Guu, Kelvin, Miller, John, and Liang, Percy.
\newblock Traversing knowledge graphs in vector space.
\newblock \emph{arXiv preprint arXiv:1506.01094}, 2015.

\bibitem[He et~al.(2015)He, Liu, Ji, and Zhao]{he2015learning}
He, Shizhu, Liu, Kang, Ji, Guoliang, and Zhao, Jun.
\newblock Learning to represent knowledge graphs with gaussian embedding.
\newblock In \emph{Proceedings of the 24th ACM International on Conference on
  Information and Knowledge Management}, pp.\  623--632. ACM, 2015.

\bibitem[Hofstadter(2001)]{hofstadter2001analogy}
Hofstadter, Douglas~R.
\newblock Analogy as the core of cognition.
\newblock \emph{The analogical mind: Perspectives from cognitive science}, pp.\
   499--538, 2001.

\bibitem[Holyoak et~al.(1996)Holyoak, Holyoak, and Thagard]{holyoak1996mental}
Holyoak, Keith~J, Holyoak, Keith~James, and Thagard, Paul.
\newblock \emph{Mental leaps: Analogy in creative thought}.
\newblock MIT press, 1996.

\bibitem[Jenatton et~al.(2012)Jenatton, Roux, Bordes, and
  Obozinski]{jenatton2012latent}
Jenatton, Rodolphe, Roux, Nicolas~L, Bordes, Antoine, and Obozinski,
  Guillaume~R.
\newblock A latent factor model for highly multi-relational data.
\newblock In \emph{Advances in Neural Information Processing Systems}, pp.\
  3167--3175, 2012.

\bibitem[Ji et~al.(2015)Ji, He, Xu, Liu, and Zhao]{ji2015knowledge}
Ji, Guoliang, He, Shizhu, Xu, Liheng, Liu, Kang, and Zhao, Jun.
\newblock Knowledge graph embedding via dynamic mapping matrix.
\newblock In \emph{ACL (1)}, pp.\  687--696, 2015.

\bibitem[Ji et~al.(2016)Ji, Liu, He, and Zhao]{DBLP:conf/aaai/JiLH016}
Ji, Guoliang, Liu, Kang, He, Shizhu, and Zhao, Jun.
\newblock Knowledge graph completion with adaptive sparse transfer matrix.
\newblock In \emph{Proceedings of the Thirtieth {AAAI} Conference on Artificial
  Intelligence, February 12-17, 2016, Phoenix, Arizona, {USA.}}, pp.\
  985--991, 2016.
\newblock URL
  \url{http://www.aaai.org/ocs/index.php/AAAI/AAAI16/paper/view/11982}.

\bibitem[Lin et~al.(2015{\natexlab{a}})Lin, Liu, Luan, Sun, Rao, and
  Liu]{lin2015modeling}
Lin, Yankai, Liu, Zhiyuan, Luan, Huanbo, Sun, Maosong, Rao, Siwei, and Liu,
  Song.
\newblock Modeling relation paths for representation learning of knowledge
  bases.
\newblock \emph{arXiv preprint arXiv:1506.00379}, 2015{\natexlab{a}}.

\bibitem[Lin et~al.(2015{\natexlab{b}})Lin, Liu, Sun, Liu, and
  Zhu]{lin2015learning}
Lin, Yankai, Liu, Zhiyuan, Sun, Maosong, Liu, Yang, and Zhu, Xuan.
\newblock Learning entity and relation embeddings for knowledge graph
  completion.
\newblock In \emph{AAAI}, pp.\  2181--2187, 2015{\natexlab{b}}.

\bibitem[Liu \& Yang(2015)Liu and Yang]{liu2015bipartite}
Liu, Hanxiao and Yang, Yiming.
\newblock Bipartite edge prediction via transductive learning over product
  graphs.
\newblock In \emph{ICML}, pp.\  1880--1888, 2015.

\bibitem[Liu \& Yang(2016)Liu and Yang]{liu2016cross}
Liu, Hanxiao and Yang, Yiming.
\newblock Cross-graph learning of multi-relational associations.
\newblock In \emph{Proceedings of The 33rd International Conference on Machine
  Learning}, pp.\  2235--2243, 2016.

\bibitem[Mikolov et~al.(2013)Mikolov, Sutskever, Chen, Corrado, and
  Dean]{mikolov2013distributed}
Mikolov, Tomas, Sutskever, Ilya, Chen, Kai, Corrado, Greg~S, and Dean, Jeff.
\newblock Distributed representations of words and phrases and their
  compositionality.
\newblock In \emph{Advances in neural information processing systems}, pp.\
  3111--3119, 2013.

\bibitem[Min et~al.(2013)Min, Grishman, Wan, Wang, and Gondek]{min2013distant}
Min, Bonan, Grishman, Ralph, Wan, Li, Wang, Chang, and Gondek, David.
\newblock Distant supervision for relation extraction with an incomplete
  knowledge base.
\newblock In \emph{HLT-NAACL}, pp.\  777--782, 2013.

\bibitem[Minsky(1988)]{minsky1988society}
Minsky, Marvin.
\newblock \emph{Society of mind}.
\newblock Simon and Schuster, 1988.

\bibitem[Nguyen et~al.(2016)Nguyen, Sirts, Qu, and Johnson]{nguyen2016stranse}
Nguyen, Dat~Quoc, Sirts, Kairit, Qu, Lizhen, and Johnson, Mark.
\newblock Stranse: a novel embedding model of entities and relationships in
  knowledge bases.
\newblock \emph{arXiv preprint arXiv:1606.08140}, 2016.

\bibitem[Nickel et~al.(2011)Nickel, Tresp, and Kriegel]{nickel2011three}
Nickel, Maximilian, Tresp, Volker, and Kriegel, Hans-Peter.
\newblock A three-way model for collective learning on multi-relational data.
\newblock In \emph{Proceedings of the 28th international conference on machine
  learning (ICML-11)}, pp.\  809--816, 2011.

\bibitem[Nickel et~al.(2015)Nickel, Murphy, Tresp, and
  Gabrilovich]{nickel2015review}
Nickel, Maximilian, Murphy, Kevin, Tresp, Volker, and Gabrilovich, Evgeniy.
\newblock A review of relational machine learning for knowledge graphs.
\newblock \emph{arXiv preprint arXiv:1503.00759}, 2015.

\bibitem[Nickel et~al.(2016)Nickel, Rosasco, and
  Poggio]{DBLP:conf/aaai/NickelRP16}
Nickel, Maximilian, Rosasco, Lorenzo, and Poggio, Tomaso~A.
\newblock Holographic embeddings of knowledge graphs.
\newblock In \emph{Proceedings of the Thirtieth {AAAI} Conference on Artificial
  Intelligence, February 12-17, 2016, Phoenix, Arizona, {USA.}}, pp.\
  1955--1961, 2016.
\newblock URL
  \url{http://www.aaai.org/ocs/index.php/AAAI/AAAI16/paper/view/12484}.

\bibitem[Pennington et~al.(2014)Pennington, Socher, and
  Manning]{pennington2014glove}
Pennington, Jeffrey, Socher, Richard, and Manning, Christopher~D.
\newblock Glove: Global vectors for word representation.
\newblock In \emph{EMNLP}, volume~14, pp.\  1532--1543, 2014.

\bibitem[Plate(2003)]{plate2003holographic}
Plate, Tony~A.
\newblock Holographic reduced representation: Distributed representation for
  cognitive structures.
\newblock 2003.

\bibitem[Recht et~al.(2011)Recht, Re, Wright, and Niu]{recht2011hogwild}
Recht, Benjamin, Re, Christopher, Wright, Stephen, and Niu, Feng.
\newblock Hogwild: A lock-free approach to parallelizing stochastic gradient
  descent.
\newblock In \emph{Advances in Neural Information Processing Systems}, pp.\
  693--701, 2011.

\bibitem[Singhal(2012)]{singhal2012introducing}
Singhal, Amit.
\newblock Introducing the knowledge graph: things, not strings.
\newblock \emph{Official google blog}, 2012.

\bibitem[Socher et~al.(2013)Socher, Chen, Manning, and Ng]{socher2013reasoning}
Socher, Richard, Chen, Danqi, Manning, Christopher~D, and Ng, Andrew.
\newblock Reasoning with neural tensor networks for knowledge base completion.
\newblock In \emph{Advances in neural information processing systems}, pp.\
  926--934, 2013.

\bibitem[Toutanova \& Chen(2015)Toutanova and Chen]{toutanova2015observed}
Toutanova, Kristina and Chen, Danqi.
\newblock Observed versus latent features for knowledge base and text
  inference.
\newblock In \emph{Proceedings of the 3rd Workshop on Continuous Vector Space
  Models and their Compositionality}, pp.\  57--66, 2015.

\bibitem[Trouillon et~al.(2016)Trouillon, Welbl, Riedel, Gaussier, and
  Bouchard]{DBLP:conf/icml/TrouillonWRGB16}
Trouillon, Th{\'{e}}o, Welbl, Johannes, Riedel, Sebastian, Gaussier,
  {\'{E}}ric, and Bouchard, Guillaume.
\newblock Complex embeddings for simple link prediction.
\newblock In \emph{Proceedings of the 33nd International Conference on Machine
  Learning, {ICML} 2016, New York City, NY, USA, June 19-24, 2016}, pp.\
  2071--2080, 2016.
\newblock URL \url{http://jmlr.org/proceedings/papers/v48/trouillon16.html}.

\bibitem[Turney(2008)]{turney2008latent}
Turney, Peter~D.
\newblock The latent relation mapping engine: Algorithm and experiments.
\newblock \emph{Journal of Artificial Intelligence Research}, 33:\penalty0
  615--655, 2008.

\bibitem[Wang et~al.(2014)Wang, Zhang, Feng, and Chen]{wang2014knowledge}
Wang, Zhen, Zhang, Jianwen, Feng, Jianlin, and Chen, Zheng.
\newblock Knowledge graph embedding by translating on hyperplanes.
\newblock In \emph{AAAI}, pp.\  1112--1119. Citeseer, 2014.

\bibitem[Wilkinson \& Wilkinson(1965)Wilkinson and
  Wilkinson]{wilkinson1965algebraic}
Wilkinson, James~Hardy and Wilkinson, James~Hardy.
\newblock \emph{The algebraic eigenvalue problem}, volume~87.
\newblock Clarendon Press Oxford, 1965.

\bibitem[Xie et~al.(2016)Xie, Liu, and Sun]{xie2016representation}
Xie, Ruobing, Liu, Zhiyuan, and Sun, Maosong.
\newblock Representation learning of knowledge graphs with hierarchical types.
\newblock In \emph{Proceedings of the Twenty-Fifth International Joint
  Conference on Artificial Intelligence}, pp.\  2965--2971, 2016.

\bibitem[Yang et~al.(2014)Yang, Yih, He, Gao, and
  Deng]{DBLP:journals/corr/YangYHGD14a}
Yang, Bishan, Yih, Wen{-}tau, He, Xiaodong, Gao, Jianfeng, and Deng, Li.
\newblock Embedding entities and relations for learning and inference in
  knowledge bases.
\newblock \emph{CoRR}, abs/1412.6575, 2014.
\newblock URL \url{http://arxiv.org/abs/1412.6575}.

\bibitem[Yang \& Liu(1999)Yang and Liu]{yang1999re}
Yang, Yiming and Liu, Xin.
\newblock A re-examination of text categorization methods.
\newblock In \emph{Proceedings of the 22nd annual international ACM SIGIR
  conference on Research and development in information retrieval}, pp.\
  42--49. ACM, 1999.

\end{thebibliography}
\bibliographystyle{icml2017}

\end{document}